\documentclass[10pt,twocolumn,letterpaper]{article}

\usepackage{cvpr}              
\usepackage{adjustbox}

\usepackage{multirow}

\usepackage{amsthm}
\usepackage{mathtools}

\usepackage{stmaryrd}

\usepackage{pifont}       
\usepackage{bbding}       
\usepackage{fontawesome}  %
\usepackage{titletoc}
\usepackage{bm}
\newcommand{\cmark}{\ding{51}}
\newcommand{\xmark}{\ding{55}}

\definecolor{cvprblue}{rgb}{0.21,0.49,0.74}
\usepackage[pagebackref,breaklinks,colorlinks,allcolors=cvprblue]{hyperref}

\def\paperID{2196} 
\def\confName{CVPR}
\def\confYear{2025}

\title{OSV: One Step is Enough for High-Quality Image to Video
Generation}

\author{
Xiaofeng Mao$^{1\ast}$~~~ Zhengkai Jiang$^{2\ast\dagger}$~~~Fu-yun Wang$^{3\ast}$~~~Jiangning Zhang$^{4,5}$~~~ \\ 
Hao Chen$^{1}$~~~Mingmin Chi$^{1\ddagger}$~~~Yabiao Wang$^{4,5\ddagger}$~~~Wenhan Luo$^{2}$ \\
$^1$\normalfont{Fudan University}~~~$^2$\normalfont{Hong Kong University of Science and Technology}~~~\\
$^3$\normalfont{The Chinese University of Hong Kong}~~~$^4$\normalfont{Zhejiang University}~~~$^5$\normalfont{Tencent YouTu Lab} \\
\small{$\{$xfmao23, haochen22$\}$@fudan.edu.cn~~~mmchi@fudan.edu.cn} \\
\small{$\{$zkjiang,whluo$\}$@ust.hk~~~fywang@link.cuhk.edu.hk~~~vtzhang@tencent.com~~~yabiaowang@zju.edu.cn} \\
}

\begin{document}
\maketitle

\if TT\insert\footins{\noindent\footnotesize\\
    $^{\ast}$~Equal Contribution \\
    $^{\dagger}$~Project Leader \\
    $^{\ddagger}$~Corresponding Author~(This work was supported  by the Natural Science Foundation of China under contract 62171139). \\

}\fi

\begin{abstract}

Video diffusion models have shown great potential in generating high-quality videos, making them an increasingly popular focus. However, their inherent iterative nature leads to substantial computational and time costs. Although techniques such as consistency distillation and adversarial training have been employed to accelerate video diffusion by reducing inference steps, these methods often simply transfer the generation approaches from Image diffusion models to video diffusion models. As a result, these methods frequently fall short in terms of both performance and training stability. In this work, we introduce a two-stage training framework that effectively combines consistency distillation with adversarial training to address these challenges. Additionally, we propose a novel video discriminator design, which eliminates the need for decoding the video latents and improves the final performance. 
Our model is capable of producing high-quality videos in merely one-step, with the flexibility to perform multi-step refinement for further performance enhancement. Our quantitative evaluation on the OpenVid-1M benchmark shows that our model significantly outperforms existing methods. Notably, our 1-step performance~(FVD $171.15$) exceeds the 8-step performance of the  consistency distillation based method, AnimateLCM~(FVD $184.79$), and approaches the 25-step performance of advanced Stable Video Diffusion~(FVD $156.94$).

\end{abstract}
\section{Introduction}

Video synthesis provides rich visual effects and creative expression for films, television, advertisements, and games. Diffusion models are playing an increasingly important role in video synthesis~\cite{blattmann2023align,shi2024motion, svd,make-a-video,ho2022video,he2022latent,li2023videogen}. Typically, diffusion models involve a forward process and a reverse process. In the forward process, real data is iteratively perturbed with noise until it converges to a simple noise distribution, typically Gaussian. In the reverse process, the noise is gradually removed, ultimately transitioning back to the target data distribution. However, this reverse process usually requires the numerical solution of a generative ODE, termed Probability Flow ODE (PF-ODE)~\cite{sde}. The iterative nature of this numerical solving process leads to significantly higher computational costs compared to other generative models~(e.g., GANs)~\cite{gan,yin2024one,yin2024improved}. These computational demands are even more significant in video synthesis; for instance, generating a short 2-second video clip using Stable Video Diffusion~(SVD)~\cite{svd} on a high-performance A100 GPU can take over 30 seconds.

\begin{figure}[t]
\centering
\includegraphics[width=0.7\linewidth]{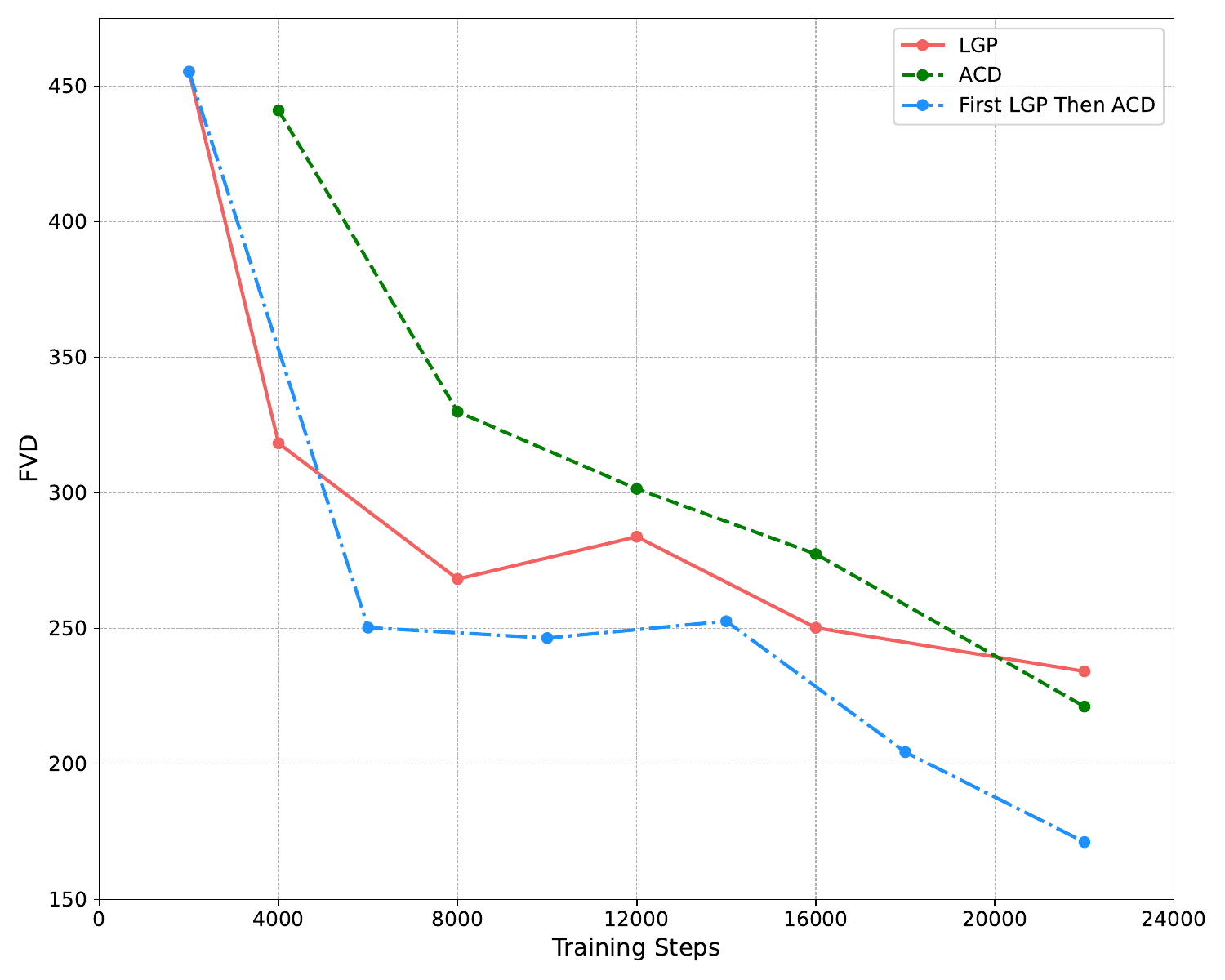}
\caption{OSV is a two-stage video diffusion acceleration strategy. In the first stage, GAN is applied for better training efficiency. In the second stage, we apply consistency distillation to boost the performance upper-bound.}
\label{fig:teaser1}
\end{figure}

\begin{figure*}[!t]
  \centering
  \includegraphics[width=\linewidth]{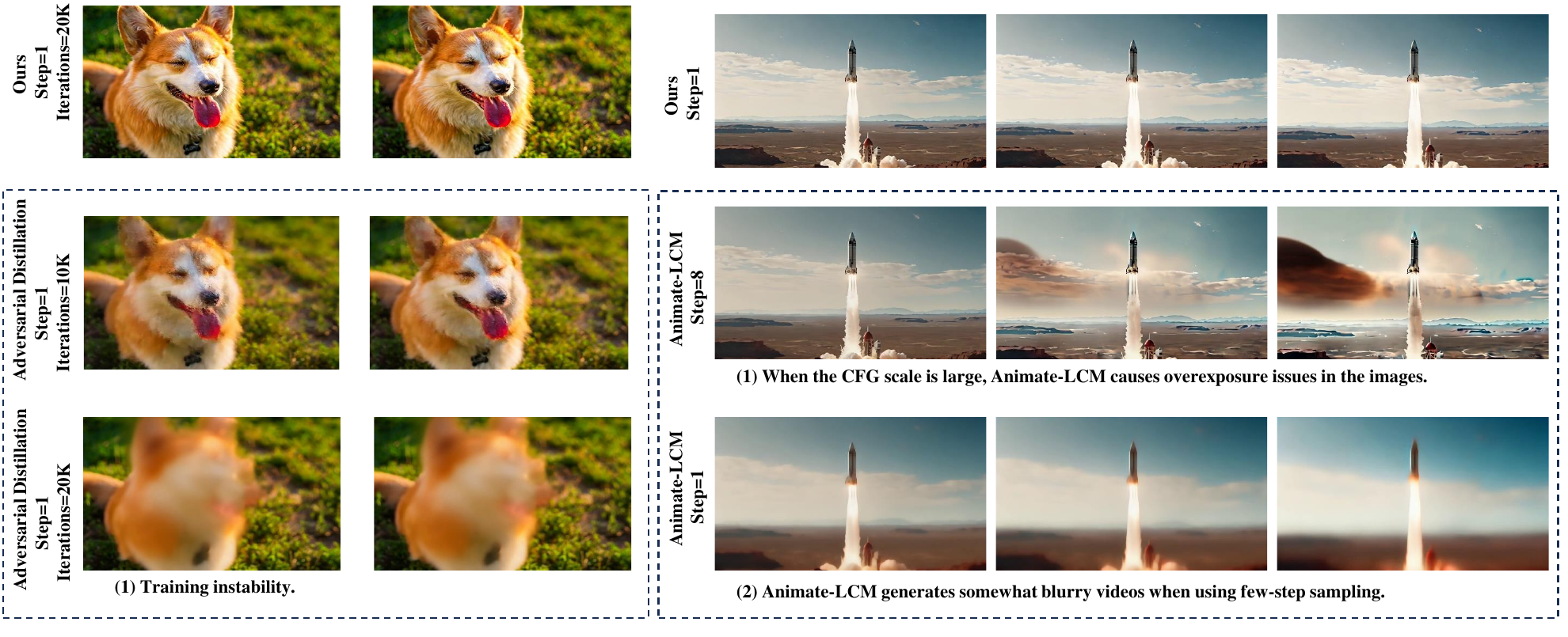}
 \caption{Summative motivation. We observe and summarize crucial limitations for (latent) consistent models and generalize to the design space, which are well addressed by our approach.
 }
  \label{fig:compare}
\end{figure*}

Currently, several strategies have been proposed to reduce the computational costs associated with video generation. Some approaches utilize consistency distillation in the latent space~(LCM)~\cite{luo2023latent,song2023consistency} for acceleration; however, they often struggle to achieve competitive results in few-step settings, such as one or two steps. Other methods initialize with pre-trained diffusion models and incorporate adversarial loss in a GAN-based framework to accelerate generation~(GAN). Nonetheless, these methods frequently suffer from training instability. Moreover, since advanced diffusion models are typically trained in the latent space~(i.e., they apply an auto-encoder to encode the high-resolution videos into the latent space to facilitate the training), GAN-based methods require decoding latent features into the real image or video space before passing them to the discriminator. This process incurs substantial memory usage and computational overhead, particularly in high-resolution generation tasks. Figure\textcolor{red}{~\ref{fig:compare}} illustrates our summative observations of previous methods, which we summarize as follows: 

\noindent \textbf{Exposure Issues}: Although LCM can accept classifier-free guidance~(CFG)~\cite{ho2021classifier}, higher CFG values can lead to exposure problems~\cite{wang2024phased}. This complicates the selection of hyperparameters for training.

\noindent \textbf{Efficiency}: LCM faces slow training convergence and often produces poor results, particularly when with less than four inference steps, which limits sampling efficiency. While introducing GANs can alleviate some of the efficiency issues of LCM, it introduces new problems.

\noindent \textbf{Increased Training Burden}: 
Because advanced diffusion models are typically trained in the latent space—where high-resolution videos are encoded via an auto-encoder—GAN-based approaches need to decode these latent features back into the real image or video space before they can be evaluated by the discriminator~\cite{sauer2023adversarial}. This decoding step significantly increases memory consumption and computational load, particularly when dealing with high-resolution or long video-length generation tasks.   

\noindent \textbf{Training Instability}: GANs are known for training instability. Introducing GANs can easily lead to training instability, causing the model to degrade as training progresses.

\noindent \textbf{Inability to Iterative Refine}: GANs training does not allow for iterative refinement of generated results, unlike consistency models, leading to poorer results with more sampling iterations.

Moreover, most existing acceleration methods for video diffusion models are derived from image diffusion models (or minor modifications such as adding temporal discriminator heads). For example, Animate-LCM~\cite{wang2024animatelcm} draws inspiration from LCM, and SF-V~\cite{zhang2024sf} is influenced by UFOGen~\cite{xu2024ufogen}. Directly transferring methods such as ADD~\cite{sauer2023adversarial}, LADD~\cite{sauer2024fast}, and UFOGen to video diffusion models presents several issues, including a \textbf{significant increase in GPU memory usage} and the potential for model collapse due to adversarial training, which may lead to \textbf{extremely weak actions in the generated videos}.

In this work, we present \textbf{OSV}~(\textbf{O}ne \textbf{S}tep \textbf{V}ideo Generation), allowing for high-quality image-to-video generation in one step while still supporting multi-step refinement. OSV is a two-stage video diffusion acceleration training strategy. Specifically, in the first stage, we incorporate Low-Rank Adaptation~(LoRA)~\cite{hu2021lora} and fully utilize GANs training, with real data as the true condition for the GANs, greatly accelerating model convergence. In the second stage, we introduce the LCM training function, with data generated by the teacher model serving as the true condition for the GANs, while only fine-tuning specific layers of the network. To further facilitate the training convergence, we replace the commonly applied one-step ODE solver with multi-step solving, ensuring higher distillation accuracy. The second training stage addresses GANs training instability in the later training phase and further improves the performance upper-bound with the knowledge transferred from the teacher video diffusion models. To further improve the training efficiency, we revise the current popular discriminator designs and propose to discard the VAE decoder in the adversarial training. As illustrated in Figure\textcolor{red}{~\ref{fig:compare_gan}}, we replace the VAE decoder with a simple up-sampling operator. That is, we directly feed the upsampled video latent into the discriminator whose backbone is pre-trained on the real image/video space~(DINOv2~\cite{oquab2023dinov2}). Our ablation study shows it not only reduces the training cost but also achieves better performance. Our proposed two-stage training not only facilitates early stage training efficiency but also stablizes training and improves the performance upper-bound in the later stage. As shown in Figure\textcolor{red}{~\ref{fig:teaser1}} verifies our claim.

In summary, we investigate the limitations of previous consistency model-based and GAN training-based methods and propose a two-stage training approach, combining the strength of both and achieving state-of-the-art performance on fast image-to-video generation.

\section{Related Works}

\noindent\textbf{Diffusion Distillation.}
Denoising process usually has many steps, making them 2-3 orders of magnitude slower than other generative models such as GANs and
VAEs. Recent progress on diffusion acceleration has focused on speeding up iteratively time-consuming generation process through distillation~\cite{song2023consistency,song2023improved,salimans2022progressive, heek2024multistep,yan2024perflow,ren2024hyper,zheng2024trajectory,sauer2023adversarial,sauer2024fast,kim2023consistency,zhou2024score,wang2024phased,li2024t2v}. Typically, they train a generator to approximate the ordinary differential equation (ODE) sampling trajectory of the teacher model, resulting in fewer sampling steps. Particularly, Progressive Distillation~\cite{salimans2022progressive,meng2023distillation} trains the student to predict directions pointing to the next flow locations. ADD~\cite{sauer2023adversarial} leverages an adversarial loss to ensure high-fidelity image generation. SDXL-Lighting~\cite{lin2024sdxl} combines progressive distillation and adversarial distillation, striking a balance between mode coverage and quality. 

When it comes to video distillation, Video-LCM~\cite{wang2023videolcm} builds upon existing latent video diffusion and incorporates consistency distillation techniques, achieving high-fidelity and smooth video synthesis with only four sampling steps. Furthermore, Animate-LCM~\cite{wang2024animatelcm} proposes a decoupled consistency learning strategy that separates the distillation of image generation priors and motion generation priors, enhancing visual quality and training efficiency. SF-V~\cite{zhang2024sf} follows diffusion-as-GAN paradigms and proposes single-step video generation by leveraging adversarial training on the SVD~\cite{blattmann2023stable} model. In contrast, our empirical observations show that decoupling adversarial latent training and consistency distillation improves training stability and generation quality. Additionally, instead of using UNet itself as the discriminator feature encoder, we use DINOv2~\cite{oquab2023dinov2} as a feature extractor, which both improves training efficiency and generation quality. Moreover, a multi-step consistency distillation training strategy is also proposed.
\section{Preliminaries}
\textbf{Diffusion Models. } Diffusion mdoel~\cite{ddpm,sde} gradually introduces random noise through the diffusion process, transforming the current state ${\mathbf {x}_0}$ into a previous state ${\mathbf {x}_t}$. 


We consider the continuous case of diffusion models. The forward process of a diffusion model can be described as:
\begin{equation}
d {\mathbf {x}} = {f}_t({\mathbf {x}}) \, dt + g_t \, d\boldsymbol{w}. \label{eq:sde-forward}
\end{equation}
where $f_{t}(\mathbf {x}) = \frac{ d \log \alpha_{t}}{ d{t}}\mathbf {x}$ and $g_{t}^2 = \frac{\mathrm d \sigma_t^2}{ d t} - 2 \frac{ d \log \alpha_t}{ d t} \sigma_{t}^2$, 
${\alpha}_t$ is the predefined scale factor and $\boldsymbol w_{t}$ denoting the standard Winer process. $\sigma_{t}$ controls the level of noise.

Considering the reverse process of diffusion models in the continuous case:
\begin{equation}
d \mathbf {x} = \left( {f}_t(\mathbf {x}) - \frac{1}{2} g_t^2 \nabla_{\mathbf {x}} \log p_t(\mathbf {x}) \right) \, dt. \label{eq:flow-ode}
\end{equation}This is known as the Probability Flow Ordinary Differential Equation (PF-ODE). We use a neural network \( \epsilon_\theta (\mathbf {x}_t, t) \) to approximate \( \nabla_{\mathbf {x}} \log p_t(\mathbf {x}) \).

\noindent\textbf{Consistency Models. }Consistency Models~\cite{song2023consistency} are built upon the PF-ODE in continuous-time diffusion models. Given a PF-ODE that smoothly transforms data into noise, Consistency Models learn to map any point to the initial point of the trajectory at any time step for generative modeling. The formula can be described as:
\begin{equation}
{f} : \left(\mathbf {x}_t, t\right) \mapsto \mathbf {x}_\kappa, \quad t \in [\kappa, T].
\end{equation}

$\kappa$ is a number greater than 0 but close to 0. To ensure the boundary conditions hold for any consistent functions, Consistency Models typically employ skip connections. Suppose we have a free-form deep neural network $F_\theta$, which can be formulated as:
\begin{equation}
f_\theta(\mathbf {x}, t) = c_\text{skip}(t) \mathbf {x} + c_\text{out}(t) F_\theta(\mathbf {x}, t),
\end{equation}
where \( c_\text{skip}(t) \) and \( c_\text{out}(t) \) are differentiable functions such that \( c_\text{skip}(\kappa) = 1 \), and \( c_\text{out}(\kappa) = 0 \), $F_\theta(\mathbf {x}, t)$ represents the output of the neural network.

For training Consistency Models, the output is enforced to be the same for any pair belonging to the same PF-ODE trajectory, i.e., \( f(\mathbf{x}_t, t) = f(\mathbf{x}_{t'}, t') \) for all \( t, t' \in [\kappa, T] \). To maintain training stability, an Exponential Moving Average (EMA) of the target model is used, given by:

\begin{equation}
\resizebox{\linewidth}{!}{
$
\mathcal{L}_\text{CD}^N(\theta, \theta^{-}; \phi) = \mathbb{E}\left[ \lambda(t_n) d\left(f_\theta(\mathbf{x}_{t_{n+1}}, t_{n+1}), f_{\theta^-}({\mathbf{x}}_{t_n}^\phi, t_n)\right)\right],
$
}
\end{equation}
where $\lambda(t_n)$ is a weighting function, $d(\cdot, \cdot)$ is a distance metric and ${\mathbf {x}}_{t_n}^\phi = \mathbf {x}_{t_{n+1}} + (t_n - t_{n+1})\Phi(\mathbf {x}_{t_{n+1}}, t_{n+1}; \phi)$. $\Phi(\cdots; \phi)$ represents the update function of a one-step ODE solver applied to the empirical PF-ODE.  $\theta^{-}$ is the EMA version weight of target models. 

\begin{figure}
    \centering
    \includegraphics[width=\linewidth]{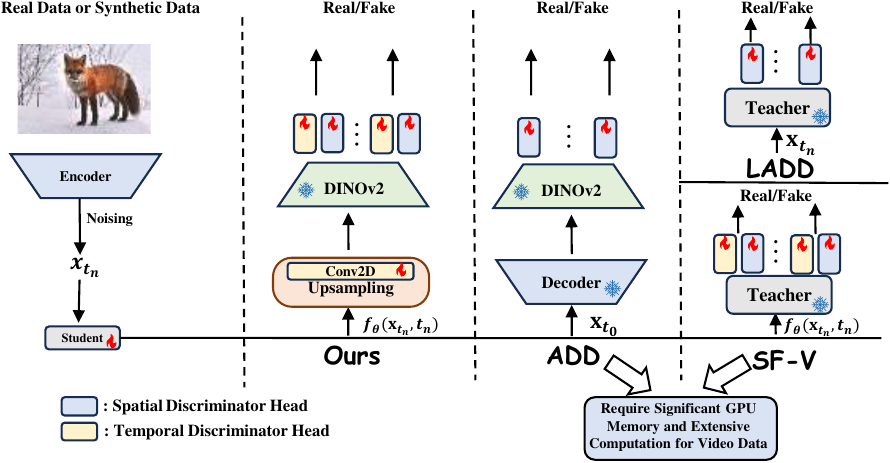}
    \caption{Comparison of Different Adversarial Training Methods. SF-V requires the encoder of UNet as the feature extraction part of the discriminator. ADD perform adversarial distillation on raw image pixel, which needs to convert latent to image thorough VAE Decoder. In contrast, we directly upsample the latent signal, replacing the decoder with a simple upsampling layer. Only this modification results in a significant speedup in training on NVIDIA H800 at a resolution of 512$\times$512, reducing the average iteration time from \textcolor{blue}{4.29} seconds to \textcolor{blue}{2.61} seconds, and also decreases the occurrence of floating-point overflows during half-precision training. In addition, OSV training consumes \textcolor{blue}{35.8} gigabytes (GB) of GPU memory, a substantial reduction compared to SF-V's \textcolor{blue}{73.5} gigabytes (GB) requirement.}
    \label{fig:compare_gan}
\end{figure}

\begin{figure*}[!t]
  \centering
  \includegraphics[width=\linewidth]{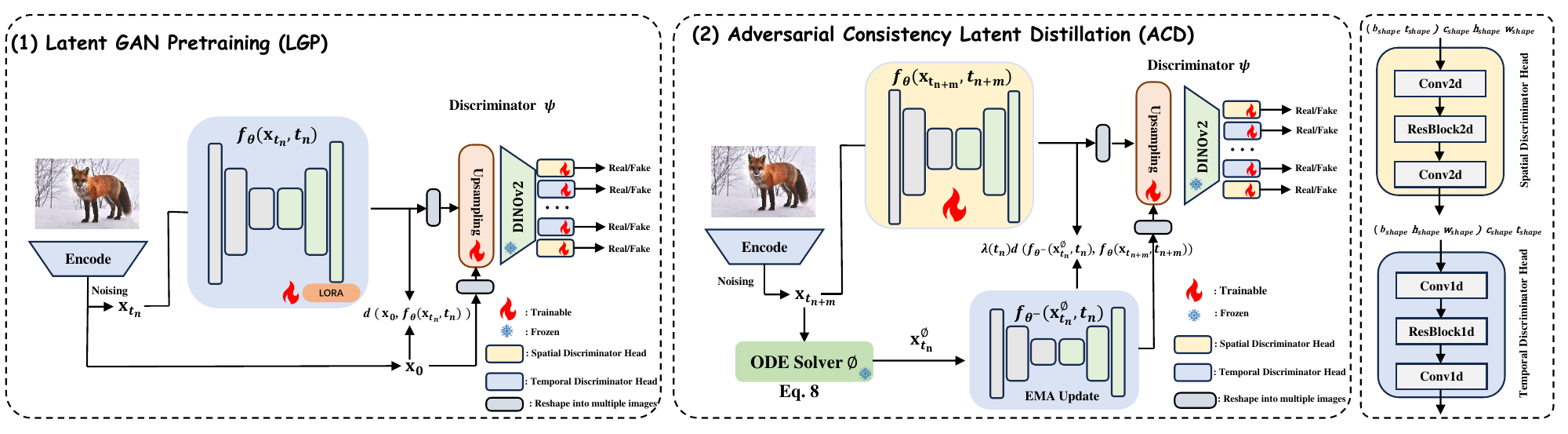}
  \caption{Overview of OSV. In the first stage, we combine GAN loss and Huber loss~\cite{song2023improved} for better training efficiency. In the second stage, we use consistency distillation loss to boost the performance upper-bound. $b_{sahpe}$, $t_{shape}$, $c_{shape}$, $h_{shape}$ and $w_{shape}$ represent the batch size, number of frames, color channels, height, and width of the input video, respectively.}
  \label{fig:main}
\end{figure*}

\section{Method}

In this section, we introduce the specific technical details of our OSV model. The model employs a two-stage training method to minimize GAN training instability and incorporates a novel multi-step consistency model solver to enhance its efficiency. Observing the negative impact of CFG on the distilled model, we remove CFG and design a new high-order solver.

In this section, we detail our OSV model, which employs a two-stage training process to enhance video generation. The first stage leverages  GAN training, allowing for rapid improvement in image quality during the initial training steps. The second stage combines GAN training with consistency distillation, providing a balanced approach that further stabilizes training and enhances model performance. Finally, we introduce a novel high-order solver, which refines generation results by a high-order prediction, leading to higher accuracy and efficiency in video generation.

\noindent\textbf{Network Components. }The OSV training process consists of three main components: a student model with weights $\theta$, a EMA model with weights $\theta^{-}$, a pre-trained teacher model with frozen weights $\phi$, and a discriminator with weights $\psi$, as shown in Figure\textcolor{red}{~\ref{fig:main}}. Specifically, the student and teacher models share the same architecture, with the student model initialized from the teacher model. For the discriminator, we adopt the same structure as StyleGAN-T~\cite{sauer2023stylegan}, utilizing DINOv2~\cite{oquab2023dinov2}. We freeze the pretrained weight of DINOv2 and add trainable temporal discriminator heads and spatial discriminator heads for discrimination of the features extracted DINOv2 inspired by SF-V~\cite{zhang2024sf}. The temporal discriminator heads are composed of  1D convolution blocks. The spatial discriminators are composed of 2D convolution blocks.

\noindent\textbf{Latent GAN Pretraining. }
As shown in Figure\textcolor{red}{~\ref{fig:compare_gan}}, SF-V and ADD implement different adversarial distillation methods. The discriminator in SF-V shares the same architecture and weights as the pre-trained UNet encoder backbone and is enhanced by adding a spatial discrimination head and a temporal discrimination head after each backbone block. ADD uses DINOv2 as the discriminator. Although the discriminator in ADD reduces computational load and memory usage compared to SF-V, the student model's generated data needs to be passed through a VAE decoder before being input into the discriminator, which undoubtedly increases both computational load and memory usage. We find that using DINOv2 directly in the latent space also achieves adversarial distillation and significantly saves memory and computational resources compared to the pixel space.

We find that during the pre-training phase of the network, introducing only GAN for adversarial distillation, without LCM, can achieve rapid image quality improvement at low steps. Specifically, the student model optimizes the Huber loss and adversarial loss between the generated data $f_\theta(\mathbf{x}_{t_n}, t_n)$ and the real data $\mathbf{x}_0$, as follows:
\begin{equation}
\resizebox{0.9\linewidth}{!}{
$\mathcal{L}_{\text{OSV}}^{\text{d}_1}(\boldsymbol \theta, \boldsymbol \theta^{-}; \boldsymbol \phi; \boldsymbol \psi) $
     =$ \text{ReLU}(1 -  
D_\psi({\mathbf x}_{0})) + \text{ReLU}(1 + D_\psi(f_\theta(\mathbf{x}_{t_n}, t_n))),  $
}
\end{equation}
\begin{equation}
\resizebox{0.9\linewidth}{!}{
$\mathcal{L}_{\text{OSV}}^{\text{g}_1}(\boldsymbol \theta, \boldsymbol \theta^{-}; \boldsymbol \phi; \boldsymbol \psi) $
     = $\lambda^{LGP} * \text{ReLU}(1-D_\psi(f_\theta(\mathbf{x}_{t_n}, t_n)))+ d(\mathbf {x}_0,f_\theta(\mathbf{x}_{t_n}^\phi, t_n)), $
}
\end{equation}
where $D_\psi$ is the discriminator, $\text{ReLU}(x) = x$ if $x>0$ else $\text{ReLU}(x) = 0$ and $\lambda^{LGP}$ is a hyper-parameter. $d(x,y)=\sqrt{\parallel x-y \parallel^2_2 +c^2}-c $~\cite{song2023improved}, where $c > 0$ is an adjustable constant. Notably, during this phase, we load the student model with LoRA training and conduct a very short training period. This approach is adopted because we have observed that full-parameter training tends to result in the generation of static images, which is a consequence of model collapse caused by the input images. Using LoRA ensures that the student model retains most of the knowledge from the teacher model while facilitating rapid convergence. In fact, the first stage can be replaced by loading LoRA modules pretrained with methods such as Animate-LCM, which offers great flexibility. For a fair comparison, we train the first stage using only the method illustrated in Figure~\ref{fig:main}.

\noindent\textbf{Adversarial Consistency Latent Distillation. }
LoRA weights stay unmerged in the second stage, part of trainable parameters, and merge only at final inference. We initialize the student model and EMA model for this stage using the student model from the first stage. Similarly, we initialize the discriminator for the second stage using the discriminator from the first stage. Unlike AnimateLCM, we solve for $\mathbf{x}_{t_n}^\phi$ using the following equations:

\begin{equation}
\resizebox{\linewidth}{!}{
$\begin{aligned}
{\mathbf {x}}_{t_{n+m-1}}^\phi &= \mathbf {x}_{t_{n+m}} + (t_{n+m-1} - t_{n+m})\Phi(\mathbf {x}_{t_{n+m}}, t_{n+m}, c; \phi), \\
{\mathbf {x}}_{t_{n+m-2}}^\phi &= \mathbf {x}_{t_{n+m-1}}^\phi + (t_{n+m-2} - t_{n+m-1})\Phi(\mathbf {x}_{t_{n+m-1}}^\phi, t_{n+m-1}, c; \phi), \\
&\quad\vdots \\
{\mathbf {x}}_{t_n}^\phi &= \mathbf {x}_{t_{n+1}}^\phi + (t_{n} - t_{n+1})\Phi(\mathbf {x}_{t_{n+1}}^\phi, t_{n+1}, c; \phi),
\end{aligned}$
}
\end{equation}
where $c$ represents the condition, which in our case is an image embedding. $m$ controls the number of steps of the ODE solver. We introduce classifier-free guidance distillation similar to~\cite{luo2023latent}: $\hat{\Phi}(\mathbf{x}_{t_{n+m}}, t_{n+m}, c; \phi) = {\Phi}(\mathbf{x}_{t_{n+m}}, t_{n+m}, c_{zero}; \phi)+w*({\Phi}(\mathbf{x}_{t_{n+m}}, t_{n+m}, c; \phi)-{\Phi}(\mathbf{x}_{t_{n+m}}, t_{n+m}, c_{zero}; \phi))$, where \( c_{zero} \) represents the image embedding set to zero and $w$ controls the strength. We also optimize the adversarial loss between the student-generated data $f_\theta(\mathbf{x}_{t_{n+m}})$ and the teacher-generated data $f_{\theta^{-}}(\mathbf{x}_{t_n}^\phi, t_n)$, as follows: 
\begin{equation}
\begin{adjustbox}{max width=1\linewidth}
$\mathcal{L}_{\text{OSV}}^{\text{d}_2}(\boldsymbol \theta, \boldsymbol \theta^{-}; \boldsymbol \phi; \boldsymbol \psi) $
     =$ \text{ReLU}(1 -  
D_\psi(f_{\theta^{-}}(\mathbf{x}_{t_n}^\phi, t_n))) + \text{ReLU}(1 + D_\psi(f_\theta(\mathbf{x}_{t_{n+m}}, t_{n+m}))) \, ,$
\end{adjustbox}
\end{equation}
\begin{equation}
\begin{adjustbox}{max width=1\linewidth}
$\mathcal{L}_{\text{OSV}}^{\text{g}_2}(\boldsymbol \theta, \boldsymbol \theta^{-}; \boldsymbol \phi; \boldsymbol \psi) $
     = $\lambda^{ACD} * \text{ReLU}(1-D_\psi(f_\theta(\mathbf{x}_{t_{n+m}},t_{n+m})))+\lambda(t_n) d(f_{\theta^{-}}(\mathbf{x}_{t_n}^\phi, t_n), f_\theta(\mathbf{x}_{t_{n+m}},t_{n+m})) \, ,$
\end{adjustbox}
\end{equation}
where $\lambda^{ACD}$ is a hyper-parameter, $d(x,y)=\sqrt{\parallel x-y \parallel^2_2 +c^2}-c, c>0$ and $\lambda(t_n)$ is a weighting function, similar to \cite{song2023improved}. It is worth noting that, as shown in Figure\textcolor{red}{~\ref{fig:teaser}}, our multi-step solving method not only achieves higher accuracy than the single-step solving method with the same number of iterations but also demonstrates higher accuracy within the same training time. (When $m$ is set to 5, the time for 22,000 iterations is approximately equal to the time for 42,000 iterations when $m$ is set to 1.)

\begin{figure}[!t]
    \centering
    \includegraphics[width=0.7\linewidth]{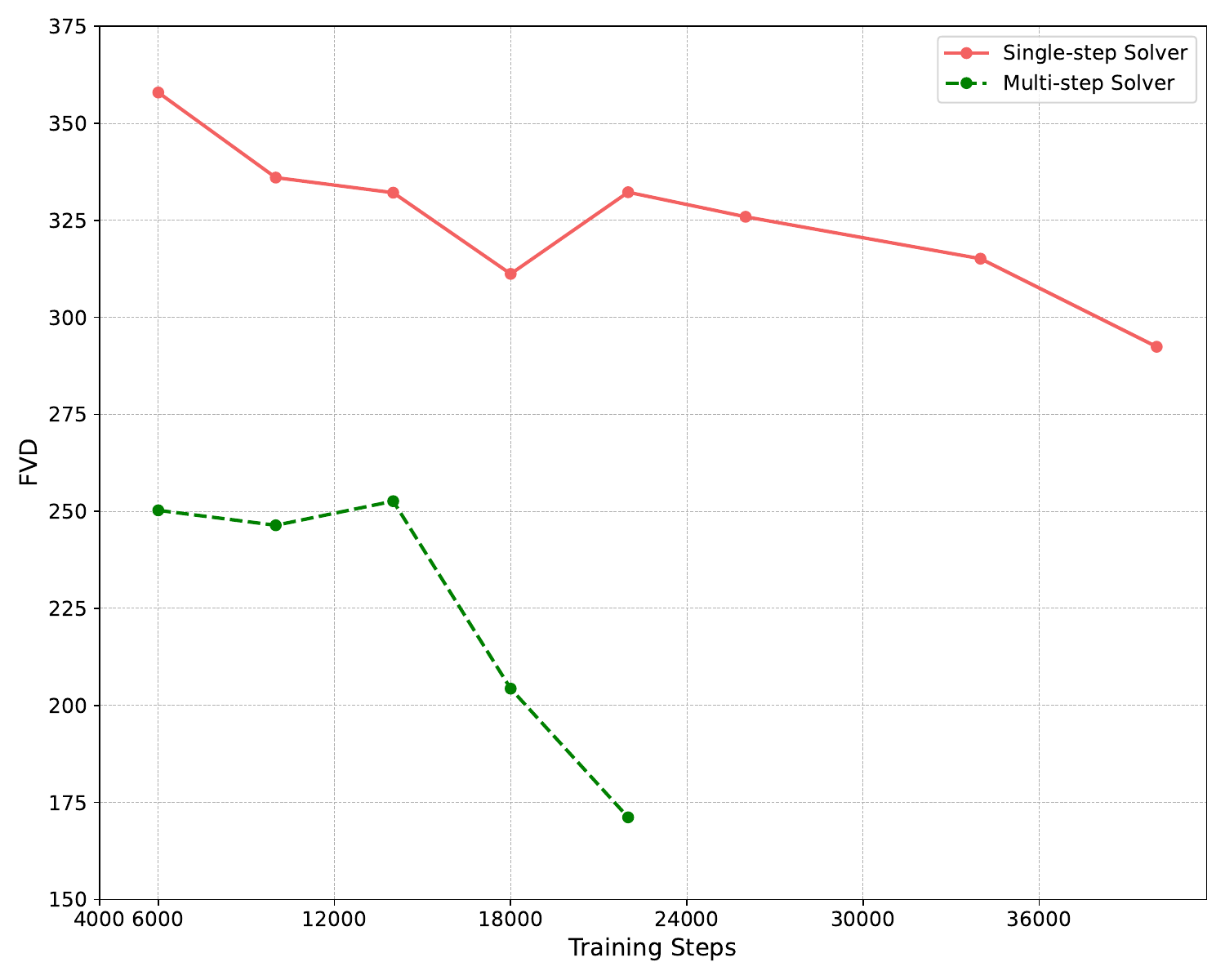}
    \caption{Effectiveness of the proposed Time Travel Sampler. Compared to one solver step, multi-step solving exhibits a faster training convergence speed and superior performance, demonstrating the effectiveness of the proposed method. $m$ is set to 5. }
    \label{fig:teaser}
    \vspace{-0.5cm}
\end{figure}

\noindent\textbf{Why Decompose into Two Stages? }LGP and ACD have fundamental differences. The training objective of LGP is to align the data distribution with the model's generated distribution. Even when the student model achieves consistency, its adversarial loss remains non-zero, thereby disrupting the consistency learning process. In contrast, ACD does not exhibit this issue. However, LGP matches noise-free data, which facilitates a rapid decrease in the loss during the early stages of training (as shown in Figure~\ref{fig:teaser1}) and enables the generation of images with less noise in a single step. We follow to the definition of PCM~\cite{wang2024phased}, letting \(\mathcal{T}_{t \to s}\) denote the flow from \(p_t\) to \(p_s\). Let \(\mathcal{T}^{\boldsymbol{\phi}}_{t \to s}\) and \(\mathcal{T}^{\boldsymbol{\theta}}_{t \to s}\) represent the transformation mappings of the ODE trajectories for the pre-trained diffusion model and our consistency model, respectively. We redefine the loss functions for LGP and ACD as follows:
\begin{equation}
\begin{adjustbox}{max width=1\linewidth}
$\mathcal{L}_{\text{LGP}}^{adv} (\boldsymbol \theta, \boldsymbol \theta^-; \boldsymbol \phi, m) =  Dis\left(\mathcal T^{\boldsymbol \theta}_{t_{n+m}\to \epsilon}\#p_{t_{n+m}}\middle\|p_{0} \right) \, ,$
\end{adjustbox}
\end{equation}
\begin{equation}
\begin{adjustbox}{max width=1\linewidth}
$\mathcal{L}_{\text{ACD}}^{adv} (\boldsymbol \theta, \boldsymbol \theta^-; \boldsymbol \phi, m) =  Dis\left(\mathcal T^{\boldsymbol \theta}_{t_{n+m}\to \epsilon}\#p_{t_{n+m}}\middle\| \mathcal T^{\boldsymbol \theta^-}_{t_{n}\to \epsilon} \mathcal T^{\boldsymbol \phi}_{t_{n+m}\to t_{n}}\# p_{t_{n+m}} \right) \, ,$
\end{adjustbox}
\end{equation}
where $\#$ is the pushforward operator, and $Dis$ is the distribution distance metric. $\mathcal{L}_{\text{LGP}}^{adv}$ is always non-zero, while $\mathcal{L}_{\text{ACD}}^{adv}$ will also converge to zero. We provide a detailed discussion in the appendix.

\noindent\textbf{High Order Sampler Based On Time Travel. }As shown in Figure\textcolor{red}{~\ref{fig:compare}}, we observe the negative impact of using CFG on the distilled model. Even with smaller CFG weights, the improvement in video generation is minimal. We decide to remove CFG and design a higher-order sampler named Time Travel Sampler (TTS). Suppose the number of sampler steps is set to $k$, corresponding to the time function $t^k$, and the number of sampler steps is set to $k+1$, corresponding to the time function $t^{k+1}$. $t^{k+1}$ has one more time step compared to $t^k$. Let $t^k_{0}=0$ and $t^{k+1}_{0}=0$, we have $t^{k+1}_{i+1} < t^k_{i+1}, t^{k+1}_{i+1}>t^k_{i}, \forall i \in [1, k-1]$. During sampling, we set the number of sampler steps to $k$. Observing that the image generation quality improves as the time step $t$ decreases, given the sample $x_{t^k_{i+1}}$, we can first solve for the sampling result at the lower time step $t^{k+1}_{i+1}$ and then revert to $t^k_{i}$ to solve the consistency function again:
\begin{equation}
\resizebox{\linewidth}{!}{
$
\begin{aligned}
&f_\theta^{t^{k}_{i+1}} = c_\text{skip}(t^{k}_{i+1}) \mathbf{x}_{t_{i+1}^k} + c_\text{out}(t^{k}_{i+1}) F_\theta(\mathbf{x}_{t^k_{i+1}}, t^{k}_{i+1}),\mathbf{x}_{t^{k+1}_{i+1}}=f_\theta^{t^{k}_{i+1}}+\sigma_{t^{k+1}_{i+1}}\hat{\epsilon}_{i+1}^{k+1}\\
&\hat{\epsilon}_{i+1}^{k+1} = (\epsilon_0 + \frac{(\mathbf{x}_{t^k_{i+1}} - f_\theta^{t^{k}_{i+1}})}{\sigma_{t^{k}_{i+1}}})/2,\\
&f_\theta^{t^{k+1}_{i+1}} = c_\text{skip}(t^{k+1}_{i+1}) \mathbf{x}_{t_{i+1}^{k+1}} + c_\text{out}(t^{k+1}_{i+1}) F_\theta(\mathbf{x}_{t^{k+1}_{i+1}}, t^{k+1}_{i+1}), \\
&\hat{\mathbf{x}}_{t^k_{i}} = c_\text{skip}(t^k_{i+1})(f_\theta^{t^{k+1}_{i+1}}+\sigma_{t^k_{i+1}}\hat{\epsilon}_{i+1}^{k}) + c_\text{out}(t^k_{i+1}) F_\theta(\mathbf{x}_{t^{k+1}_{i+1}}, t^{k+1}_{i+1})+\sigma_{t_i^k}\epsilon_0,\hat{\epsilon}_{i+1}^{k} = (\epsilon_0 + \frac{(\mathbf{x}_{t^{k+1}_{i+1}} - f_\theta^{t^{k+1}_{i+1}})}{\sigma_{t^{k+1}_{i+1}}})/2, \\
\end{aligned}$}
\end{equation}
where \(\epsilon_0\) is the random Gaussian noise. Using TTS leads to an increase in NFE, but removing CFG reduces NFE.

\begin{figure*}[!t]
  \centering
  \includegraphics[width=0.9\linewidth]{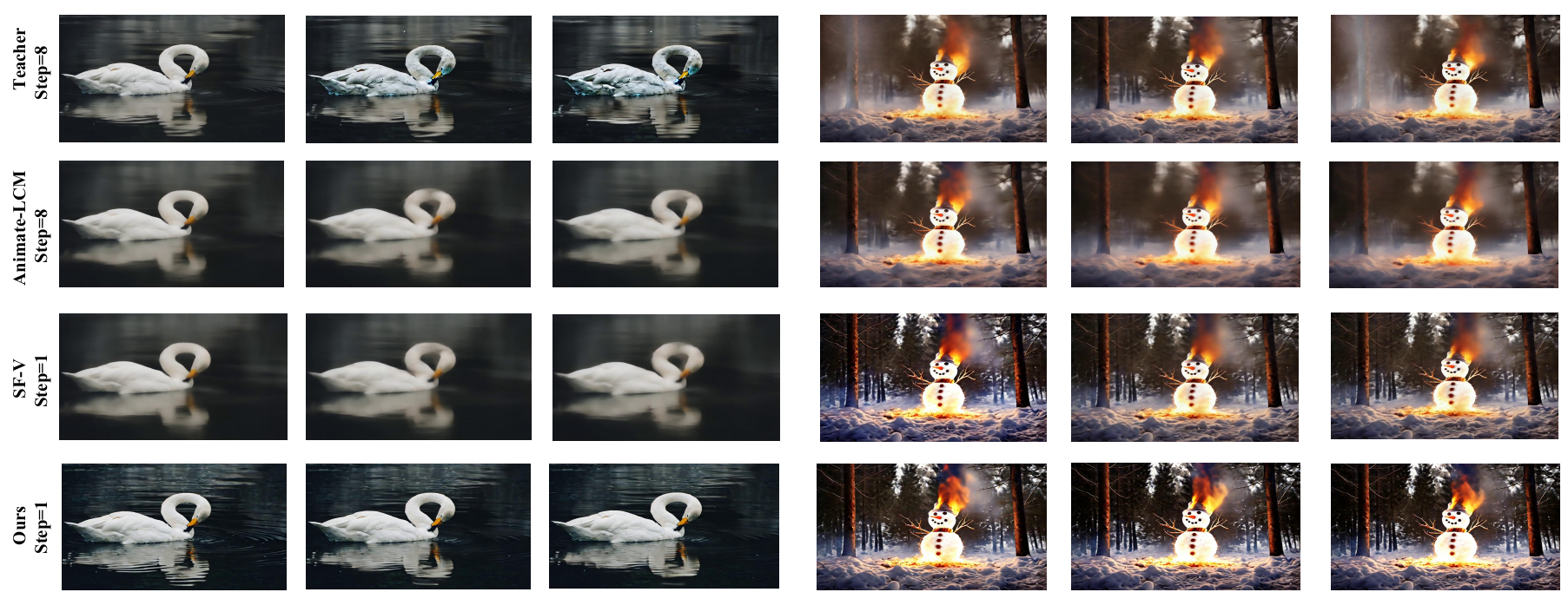}
 \caption{Qualitative generation results. One-step results of OSV with TTS achieves superior video clarity compared to other baselines. Please refer to the supplementary material for a better comparison of the videos generated by different models.}
  \label{fig:method_compare}
\end{figure*}

\section{Experiments}
\noindent\textbf{Implementation Details.} We apply stable video diffusion as the base model for most experiments. we uniformly sample 100 timesteps for training.  For better training and evaluation of our method, we utilize OpenVid-1M~\cite{nan2024openvid} for training and validation. We randomly select 1 million videos from OpenVid-1M as the training set and 1,000 videos as the test set. In the first stage, we fix the resolution of the training videos at 1024$\times$768, the FPS at 7, the batch size at 1, and the learning rate at 5e-6. The training is conducted over 2K iterations on 8 NVIDIA H800 GPUs. In the second stage, we fix the resolution of the training videos at 576$\times$320~(in the final 10K iterations, the resolution is increased to 1024$\times$768), the FPS at 7, the batch size at 1, and the learning rate at 5e-6. The training is conducted over 20K iterations on 2 NVIDIA H800 GPUs. We find that higher resolution in training datasets generally yields better results, but the impact varies across different stages, as discussed in Sec. 5.3. All stages use the Adam optimizer~\cite{kingma2014adam}. We use FVD~\cite{unterthiner2018towards} to evaluate our model. All models, including OSV, use SVD as the basic model architecture (Unet architecture), initialized with the same pretrained model, sharing no other differences with SVD except when explicitly mentioned. Apart from specifying TTS usage, we uniformly evaluate using the solver from Consistency Model.

\begin{table}[t]
  \caption{Image-to-video performance comparison on the validation set of OpenVid-1M. $\dagger$ means our implementations of SF-V~(no public weight).}
  \label{tab:1}
  \centering
  
  \begin{tabular}{cccc}
    \toprule
      Name  & Steps $\downarrow$ & NFE $\downarrow$ & FVD $\downarrow$ \\
    \midrule
    \multirow{4}*{Teacher$_\text{CFG=3.0}$~\cite{svd}} & 25 & 50 & 156.94 \\
                                      & 8  & 16 & 229.94 \\
                                      & 2  & 4  & 1015.25 \\
                                      & 1  & 2  & 1177.34 \\
    \midrule
    \multirow{4}*{AnimateLCM~\cite{wang2024animatelcm}} & 8 & 8 & 184.79 \\
                                                         & 4 & 4 & 405.80 \\
                                                         & 2 & 2 & 575.33 \\
                                                         & 1 & 1 & 997.94 \\
    \midrule
    {SF-V~\cite{zhang2024sf}}$\dagger$                       & 1 & 1 & 425.68 \\
    \midrule
    Ours & 1 & 1 & 335.36 \\
    Ours$_\text{TTS}$                   & 1 & 2 & 171.15 \\
    Ours                    & 2 & 2 & 181.95 \\
    \bottomrule
  \end{tabular}
  
\end{table}

\begin{table*}[t]
\caption{Ablation study with OSV. Unless stated otherwise, we set the sampling steps to 1 and use the Time Travel Sampler (TTS).}
\label{tab:all}
\centering
\small
\subfloat[Effect of Multi-Step Solving.]{
\centering
\setlength{\tabcolsep}{3.8mm}
\begin{tabular}{cc}
\toprule
    ODE Solver Step & FVD$\downarrow$  \\	\midrule
$1$     &332.25       \\ 
$5$ & 171.15     \\ 
\bottomrule
\end{tabular}
\label{tab:Multi-Step}
}
\hspace{1em}
\subfloat[Effect of First Stage Training.]{
\centering
\setlength{\tabcolsep}{4.0mm}
\begin{tabular}{cc}
\toprule
Stage one & FVD$\downarrow$ \\	\midrule
\xmark    & 221.75     \\ 
\cmark & 171.15       \\ 
\bottomrule
\end{tabular}
\label{tab:First_Stage}
}
\hspace{3em}
\subfloat[Effect of Second Stage Training. Note that all settings have the same NFE.]{
\centering
\setlength{\tabcolsep}{6.5mm}
\begin{tabular}{ccc}
\toprule
Stage Two &TTS & FVD$\downarrow$ \\	\midrule
\xmark &\xmark &298.35    \\
\xmark &\cmark &234.13          \\
\cmark &\cmark &171.15  \\

\bottomrule
\end{tabular}
\label{tab:Second_Stage}
}\\
\subfloat[Effect of Adversarial Training.]{
\centering
\setlength{\tabcolsep}{4.0mm}
\begin{tabular}{cc}
\toprule
Adversarial Training & FVD$\downarrow$ \\	\midrule
\xmark    &405.41  \\ 
\cmark &171.15  \\ 
\bottomrule
\end{tabular}
\label{tab:Adversarial_Training}
} 
\subfloat[Effect of VAE Decoder.]{
\centering
\setlength{\tabcolsep}{6.5mm}
\begin{tabular}{cc}
\toprule
Vae Decoder & FVD$\downarrow$ \\	\midrule
\cmark &232.25    \\ 
\xmark &171.15  \\ 
\bottomrule
\end{tabular}
\label{tab:VAE_Decoder}
}
\hspace{3em}
\subfloat[Effect of CFG.]{
\centering
\setlength{\tabcolsep}{6.5mm}
\begin{tabular}{cc}
\toprule
CFG scale & FVD$\downarrow$ \\	\midrule
3.0 &531.23    \\ 
1.5 &426.71 \\ 
No CFG &335.36 \\ 
\bottomrule
\end{tabular}
\label{tab:CFG}
}
\\
\subfloat[Effect of Data Resolution in Stage One.]{
\centering
\setlength{\tabcolsep}{6.5mm}
\begin{tabular}{cc}
\toprule
Training Data Size & FVD$\downarrow$ \\	\midrule
576$\times$320     & 455.35       \\ 
1024$\times$576    & 388.86      \\ 
\bottomrule
\end{tabular}
\label{tab:SIZE_one}
}
\hspace{1em}
\subfloat[Effect of Data Resolution in Stage Two.]{
\centering
\setlength{\tabcolsep}{6.5mm}
\begin{tabular}{cc}
\toprule
Training Data Size & FVD$\downarrow$ \\	\midrule
576$\times$320~(10K) and 1024$\times$576~(10K) &171.15       \\ 
1024$\times$576        &173.14        \\ 
\bottomrule
\end{tabular}
\label{tab:SIZE_two}
}
\hspace{2em}\end{table*}

\subsection{Quantitative Experiments}
Table\textcolor{red}{~\ref{tab:1}} illustrates the quantitative comparison of OSV with the strong baseline methods Animate-LCM and SF-V. Observing the negative impact of CFG on the distilled models, we remove CFG from both the Animate-LCM and OSV models. This reduction led to decreased inference time for the student network. OSV significantly outperforms the baseline methods, especially at low steps. Our method is compared with SF-V and AnimateLCM, using different numbers of sampling steps for each method. Additionally, we find that using the high-order sampler in a single step resulted in an FVD of 171.15, compared to an FVD of 181.95 when performing direct inference with two steps. This demonstrates better results than direct two-step inference.

Figure\textcolor{red}{~\ref{fig:user_study}} shows the results of our user study, where our model achieve higher clarity and smoothness compared to the teacher.

\subsection{Qualitative Results}
Figure\textcolor{red}{~\ref{fig:method_compare}} shows the generated results of our method in image-to-video generation, all of which achieve satisfactory performance. The generated results demonstrate that our method effectively adheres to the consistency properties across different inference steps, maintaining similar styles and motions. Other methods either suffer from overexposure issues or blurring due to motion. We exhibit good visual quality and smooth motion with only one step.

\begin{figure}[!t]
    \centering
    \includegraphics[width=0.8\linewidth]{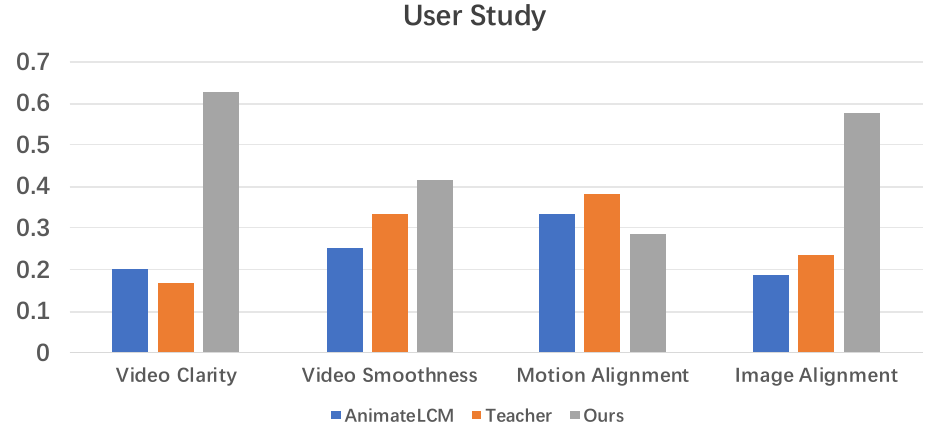}
    \caption{User study comparing our distilled model with its teacher and competing distillation baselines. For each model, we generate 30 videos across diverse scenarios, ask users to vote for the best-performing model. AnimateLCM with 8 sampling steps, SF-V with 1 sampling step, OSV with 1 sampling step and TTS, and the Teacher model with 25 sampling steps. }
    \label{fig:user_study}
\end{figure}

\subsection{Ablation Studies}

\noindent\textbf{Effect of Multi-Step Solving Method. }We set the OSV model with 5 solver steps as Baseline-1. To verify the effectiveness of the multi-step solving method, we remove the multi-step solving component and train the OSV model with 1 solver step under the same training settings. As shown in Table\textcolor{red}{~\ref{tab:Multi-Step}}, the multi-step solving method achieves higher solving accuracy. Considering the training time for all models is the same, our method also performs better, as illustrated in Figure\textcolor{red}{~\ref{fig:teaser}}.

\noindent\textbf{Effect of First Stage of Training and Second Stage of Traning. }We set the OSV model trained with both Stage 1 and Stage 2 as Baseline-2. To verify the effectiveness of Stage 1 training, we remove Stage 1 training and train the OSV model with only Stage 2 under the same training settings. As shown in Table\textcolor{red}{~\ref{tab:First_Stage}}, it is evident that performing both Stage 1 and Stage 2 training contributes more to the model's convergence. Stage 1 training primarily ensures that the student model can still generate detailed content at low steps, which aids in the consistency training of Stage 2 and prevents blurring issues similar to those in Animate-LCM. To verify the effectiveness of the second stage of training, we remove the second stage and train the OSV model with only the first stage under the same training settings. As shown in Table\textcolor{red}{~\ref{tab:Second_Stage}}, it is evident that the second stage of training contributes more to refining the generated videos. The consistency training in the second stage further enhances the details of the generated videos. We also find that our TTS sampler is effective for distilling the student model using GAN methods.

\noindent\textbf{Effect of Adversarial Training. }We set the OSV model with adversarial loss as Baseline-3. To verify the effectiveness of adversarial distillation, we remove the adversarial loss and train the OSV model with only Huber loss and consistency loss under the same training settings. As shown in Table\textcolor{red}{~\ref{tab:Adversarial_Training}}, adversarial training results in higher generation quality. Using only consistency loss leads to a fitting error between the student model and the teacher model.

\noindent\textbf{Effect of VAE Decoder. }We use Baseline-3. We add the VAE Decoder from the ADD method, as shown in Figure\textcolor{red}{~\ref{fig:compare_gan}}. As shown in Table\textcolor{red}{~\ref{tab:VAE_Decoder}}, we find that adding the VAE Decoder resulted in even worse performance. This indicates that performing adversarial training in the latent space is more beneficial for the discriminator.  As shown in Figure\textcolor{red}{~\ref{fig:vis}}, we visualize the latent space and pixel space of the input images. It can be observed that when the input images are compressed by the VAE, the outlines are preserved, retaining a significant amount of low-frequency information. This is beneficial for ViT-like models in feature extraction. We upsample the latent space data to a size suitable for DINOv2 feature extraction using sub-pixel convolution.

\noindent\textbf{Effect of CFG. }As shown in Table\textcolor{red}{~\ref{tab:CFG}}, we investigate the impact of CFG on video generation by the model. Figure\textcolor{red}{~\ref{fig:compare}} illustrates the overexposure issue caused by adding CFG. It is evident that even reducing the CFG scale still negatively affects the model. Removing CFG not only saves time but also improves the quality of the generated videos.

\noindent\textbf{Effect of Data Resolution in Stage One and Stage Two.} In the first stage, we do not introduce the consistency distillation loss, so the quality of the generated videos relates to the dataset size. If we use a dataset size of 576$\times$320, the videos are downsampled quite small, resulting in significant information loss. In the second stage, we enforce the consistency of the student model's trajectories at different time steps, so the quality of the generated videos depends on the quality of the original videos and the videos generated by the teacher model. Using a dataset size of 576$\times$320 in the early stages significantly reduces the model's distillation time and shows little difference in the FVD metric during the second stage. However, we observe that although the FVD metric shows little difference, the SVD model distilled on the low-resolution dataset encounters more failure cases during inference, such as generating videos with smaller motions. Therefore, we recommend training on a high-resolution dataset if sufficient computational resources are available.
\begin{figure}[t]
\centering
\includegraphics[width=0.41\linewidth]{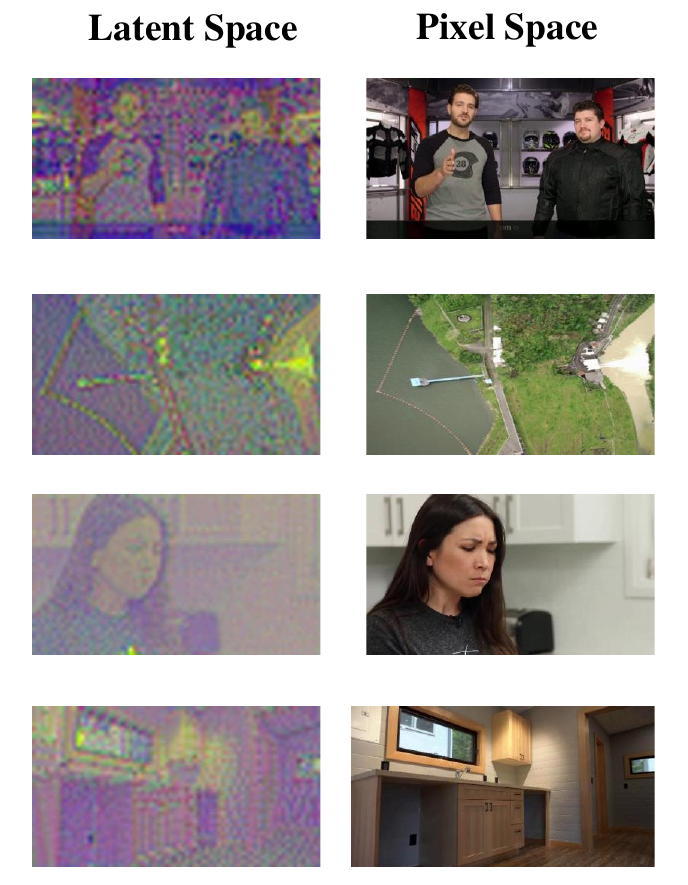}
\caption{Visualize the latent space and pixel space of the input images. To map the latent space to the RGB color space, we quantize the data within the latent space to the range [0, 255].}
\label{fig:vis}
\end{figure}

\section{Conclusion}
In this paper, we introduce the OSV, which utilizes a two-stage training process to enhance the stability and efficiency of video generation acceleration. In the first stage, we employ GAN training, achieving rapid improvements in generation quality at low steps. We propose a novel video discriminator design where we leverage pretrained image backbones~(DINOv2) and lightweight trainable temporal discriminator heads and spatial discriminator heads. We also innovatively propose to replace the commonly applied VAE decoding process with a simple up-sampling operation, which greatly facilitates training efficiency and improves model performance. The second stage integrates consistency distillation, further refining the model's performance and ensuring training stability. Additionally, we show that applying multi-step ODE solver can increase the accuracy of predictions but also facilitate faster training convergence. By removing the CFG and introducing the Time Travel Sampler~(TTS), we are able to further improve video generation quality. Our experiments demonstrate that the OSV significantly outperforms existing methods in both speed and accuracy, making it a robust and efficient solution for video generation acceleration. \\
\noindent\textbf{Limitations.} There are some bad cases when the distilled model generates human motion. For example, the distilled model produce significant blurring when attempting to generate hand movements. Increasing the number of inference steps to 4 resolve the issue. In future work, we plan to introduce stronger feature extraction networks as replacements for DINOv2. We observe that the distilled model exhibits less motion and heavily relies on the input images. This is a common phenomenon in distilled models with few-step sampling.

{
    \small
    \bibliographystyle{ieeenat_fullname}
    \bibliography{main}

\begin{thebibliography}{41}
\providecommand{\natexlab}[1]{#1}
\providecommand{\url}[1]{\texttt{#1}}
\expandafter\ifx\csname urlstyle\endcsname\relax
  \providecommand{\doi}[1]{doi: #1}\else
  \providecommand{\doi}{doi: \begingroup \urlstyle{rm}\Url}\fi

\bibitem[Blattmann et~al.(2023{\natexlab{a}})Blattmann, Dockhorn, Kulal, Mendelevitch, Kilian, Lorenz, Levi, English, Voleti, Letts, et~al.]{blattmann2023stable}
Andreas Blattmann, Tim Dockhorn, Sumith Kulal, Daniel Mendelevitch, Maciej Kilian, Dominik Lorenz, Yam Levi, Zion English, Vikram Voleti, Adam Letts, et~al.
\newblock Stable video diffusion: Scaling latent video diffusion models to large datasets.
\newblock \emph{arXiv preprint arXiv:2311.15127}, 2023{\natexlab{a}}.

\bibitem[Blattmann et~al.(2023{\natexlab{b}})Blattmann, Dockhorn, Kulal, Mendelevitch, Kilian, Lorenz, Levi, English, Voleti, Letts, et~al.]{svd}
Andreas Blattmann, Tim Dockhorn, Sumith Kulal, Daniel Mendelevitch, Maciej Kilian, Dominik Lorenz, Yam Levi, Zion English, Vikram Voleti, Adam Letts, et~al.
\newblock Stable video diffusion: Scaling latent video diffusion models to large datasets.
\newblock \emph{arXiv preprint arXiv:2311.15127}, 2023{\natexlab{b}}.

\bibitem[Blattmann et~al.(2023{\natexlab{c}})Blattmann, Rombach, Ling, Dockhorn, Kim, Fidler, and Kreis]{blattmann2023align}
Andreas Blattmann, Robin Rombach, Huan Ling, Tim Dockhorn, Seung~Wook Kim, Sanja Fidler, and Karsten Kreis.
\newblock Align your latents: High-resolution video synthesis with latent diffusion models.
\newblock In \emph{CVPR}, pages 22563--22575, 2023{\natexlab{c}}.

\bibitem[Goodfellow et~al.(2014)Goodfellow, Pouget-Abadie, Mirza, Xu, Warde-Farley, Ozair, Courville, and Bengio]{gan}
Ian Goodfellow, Jean Pouget-Abadie, Mehdi Mirza, Bing Xu, David Warde-Farley, Sherjil Ozair, Aaron Courville, and Yoshua Bengio.
\newblock Generative adversarial networks.
\newblock \emph{NeurIPS}, 2014.

\bibitem[He et~al.(2022)He, Yang, Zhang, Shan, and Chen]{he2022latent}
Yingqing He, Tianyu Yang, Yong Zhang, Ying Shan, and Qifeng Chen.
\newblock Latent video diffusion models for high-fidelity video generation with arbitrary lengths.
\newblock \emph{arXiv preprint arXiv:2211.13221}, 2\penalty0 (3):\penalty0 4, 2022.

\bibitem[Heek et~al.(2024)Heek, Hoogeboom, and Salimans]{heek2024multistep}
Jonathan Heek, Emiel Hoogeboom, and Tim Salimans.
\newblock Multistep consistency models.
\newblock \emph{arXiv preprint arXiv:2403.06807}, 2024.

\bibitem[Ho and Salimans(2021)]{ho2021classifier}
Jonathan Ho and Tim Salimans.
\newblock Classifier-free diffusion guidance.
\newblock In \emph{NeurIPS}, 2021.

\bibitem[Ho et~al.(2020)Ho, Jain, and Abbeel]{ddpm}
Jonathan Ho, Ajay Jain, and Pieter Abbeel.
\newblock Denoising diffusion probabilistic models.
\newblock \emph{NeurIPS}, 33:\penalty0 6840--6851, 2020.

\bibitem[Ho et~al.(2022)Ho, Salimans, Gritsenko, Chan, Norouzi, and Fleet]{ho2022video}
Jonathan Ho, Tim Salimans, Alexey Gritsenko, William Chan, Mohammad Norouzi, and David~J Fleet.
\newblock Video diffusion models.
\newblock \emph{arXiv:2204.03458}, 2022.

\bibitem[Hu et~al.(2021)Hu, Shen, Wallis, Allen-Zhu, Li, Wang, Wang, and Chen]{hu2021lora}
Edward~J Hu, Yelong Shen, Phillip Wallis, Zeyuan Allen-Zhu, Yuanzhi Li, Shean Wang, Lu Wang, and Weizhu Chen.
\newblock Lora: Low-rank adaptation of large language models.
\newblock \emph{arXiv preprint arXiv:2106.09685}, 2021.

\bibitem[Kim et~al.(2023)Kim, Lai, Liao, Murata, Takida, Uesaka, He, Mitsufuji, and Ermon]{kim2023consistency}
Dongjun Kim, Chieh-Hsin Lai, Wei-Hsiang Liao, Naoki Murata, Yuhta Takida, Toshimitsu Uesaka, Yutong He, Yuki Mitsufuji, and Stefano Ermon.
\newblock Consistency trajectory models: Learning probability flow ode trajectory of diffusion.
\newblock \emph{arXiv preprint arXiv:2310.02279}, 2023.

\bibitem[Kingma and Ba(2014)]{kingma2014adam}
Diederik~P Kingma and Jimmy Ba.
\newblock Adam: A method for stochastic optimization.
\newblock \emph{arXiv preprint arXiv:1412.6980}, 2014.

\bibitem[Li et~al.(2024)Li, Feng, Fu, Wang, Basu, Chen, and Wang]{li2024t2v}
Jiachen Li, Weixi Feng, Tsu-Jui Fu, Xinyi Wang, Sugato Basu, Wenhu Chen, and William~Yang Wang.
\newblock T2v-turbo: Breaking the quality bottleneck of video consistency model with mixed reward feedback.
\newblock \emph{arXiv preprint arXiv:2405.18750}, 2024.

\bibitem[Li et~al.(2023)Li, Chu, Wu, Yuan, Liu, Zhang, Li, Feng, Ding, and Wang]{li2023videogen}
Xin Li, Wenqing Chu, Ye Wu, Weihang Yuan, Fanglong Liu, Qi Zhang, Fu Li, Haocheng Feng, Errui Ding, and Jingdong Wang.
\newblock Videogen: A reference-guided latent diffusion approach for high definition text-to-video generation.
\newblock \emph{arXiv preprint arXiv:2309.00398}, 2023.

\bibitem[Lin et~al.(2024)Lin, Wang, and Yang]{lin2024sdxl}
Shanchuan Lin, Anran Wang, and Xiao Yang.
\newblock Sdxl-lightning: Progressive adversarial diffusion distillation.
\newblock \emph{arXiv preprint arXiv:2402.13929}, 2024.

\bibitem[Luo et~al.(2023)Luo, Tan, Huang, Li, and Zhao]{luo2023latent}
Simian Luo, Yiqin Tan, Longbo Huang, Jian Li, and Hang Zhao.
\newblock Latent consistency models: Synthesizing high-resolution images with few-step inference.
\newblock \emph{arXiv preprint arXiv:2310.04378}, 2023.

\bibitem[Meng et~al.(2023)Meng, Rombach, Gao, Kingma, Ermon, Ho, and Salimans]{meng2023distillation}
Chenlin Meng, Robin Rombach, Ruiqi Gao, Diederik Kingma, Stefano Ermon, Jonathan Ho, and Tim Salimans.
\newblock On distillation of guided diffusion models.
\newblock In \emph{CVPR}, pages 14297--14306, 2023.

\bibitem[Nan et~al.(2024)Nan, Xie, Zhou, Fan, Yang, Chen, Li, Yang, and Tai]{nan2024openvid}
Kepan Nan, Rui Xie, Penghao Zhou, Tiehan Fan, Zhenheng Yang, Zhijie Chen, Xiang Li, Jian Yang, and Ying Tai.
\newblock Openvid-1m: A large-scale high-quality dataset for text-to-video generation.
\newblock \emph{arXiv preprint arXiv:2407.02371}, 2024.

\bibitem[Oquab et~al.(2023)Oquab, Darcet, Moutakanni, Vo, Szafraniec, Khalidov, Fernandez, Haziza, Massa, El-Nouby, et~al.]{oquab2023dinov2}
Maxime Oquab, Timoth{\'e}e Darcet, Th{\'e}o Moutakanni, Huy Vo, Marc Szafraniec, Vasil Khalidov, Pierre Fernandez, Daniel Haziza, Francisco Massa, Alaaeldin El-Nouby, et~al.
\newblock Dinov2: Learning robust visual features without supervision.
\newblock \emph{arXiv preprint arXiv:2304.07193}, 2023.

\bibitem[Ren et~al.(2024)Ren, Xia, Lu, Zhang, Wu, Xie, Wang, and Xiao]{ren2024hyper}
Yuxi Ren, Xin Xia, Yanzuo Lu, Jiacheng Zhang, Jie Wu, Pan Xie, Xing Wang, and Xuefeng Xiao.
\newblock Hyper-sd: Trajectory segmented consistency model for efficient image synthesis.
\newblock \emph{arXiv preprint arXiv:2404.13686}, 2024.

\bibitem[Salimans and Ho(2022)]{salimans2022progressive}
Tim Salimans and Jonathan Ho.
\newblock Progressive distillation for fast sampling of diffusion models.
\newblock \emph{arXiv preprint arXiv:2202.00512}, 2022.

\bibitem[Sauer et~al.(2023{\natexlab{a}})Sauer, Karras, Laine, Geiger, and Aila]{sauer2023stylegan}
Axel Sauer, Tero Karras, Samuli Laine, Andreas Geiger, and Timo Aila.
\newblock Stylegan-t: Unlocking the power of gans for fast large-scale text-to-image synthesis.
\newblock In \emph{ICML}, pages 30105--30118. PMLR, 2023{\natexlab{a}}.

\bibitem[Sauer et~al.(2023{\natexlab{b}})Sauer, Lorenz, Blattmann, and Rombach]{sauer2023adversarial}
Axel Sauer, Dominik Lorenz, Andreas Blattmann, and Robin Rombach.
\newblock Adversarial diffusion distillation.
\newblock \emph{arXiv preprint arXiv:2311.17042}, 2023{\natexlab{b}}.

\bibitem[Sauer et~al.(2024)Sauer, Boesel, Dockhorn, Blattmann, Esser, and Rombach]{sauer2024fast}
Axel Sauer, Frederic Boesel, Tim Dockhorn, Andreas Blattmann, Patrick Esser, and Robin Rombach.
\newblock Fast high-resolution image synthesis with latent adversarial diffusion distillation.
\newblock \emph{arXiv preprint arXiv:2403.12015}, 2024.

\bibitem[Shi et~al.(2016)Shi, Caballero, Husz{\'a}r, Totz, Aitken, Bishop, Rueckert, and Wang]{shi2016real}
Wenzhe Shi, Jose Caballero, Ferenc Husz{\'a}r, Johannes Totz, Andrew~P Aitken, Rob Bishop, Daniel Rueckert, and Zehan Wang.
\newblock Real-time single image and video super-resolution using an efficient sub-pixel convolutional neural network.
\newblock In \emph{Proceedings of the IEEE conference on computer vision and pattern recognition}, pages 1874--1883, 2016.

\bibitem[Shi et~al.(2024)Shi, Huang, Wang, Bian, Li, Zhang, Zhang, Cheung, See, Qin, et~al.]{shi2024motion}
Xiaoyu Shi, Zhaoyang Huang, Fu-Yun Wang, Weikang Bian, Dasong Li, Yi Zhang, Manyuan Zhang, Ka~Chun Cheung, Simon See, Hongwei Qin, et~al.
\newblock Motion-i2v: Consistent and controllable image-to-video generation with explicit motion modeling.
\newblock In \emph{SIGGRAPH}, pages 1--11, 2024.

\bibitem[Singer et~al.(2022)Singer, Polyak, Hayes, Yin, An, Zhang, Hu, Yang, Ashual, Gafni, et~al.]{make-a-video}
Uriel Singer, Adam Polyak, Thomas Hayes, Xi Yin, Jie An, Songyang Zhang, Qiyuan Hu, Harry Yang, Oron Ashual, Oran Gafni, et~al.
\newblock Make-a-video: Text-to-video generation without text-video data.
\newblock \emph{arXiv preprint arXiv:2209.14792}, 2022.

\bibitem[Song and Dhariwal(2023)]{song2023improved}
Yang Song and Prafulla Dhariwal.
\newblock Improved techniques for training consistency models.
\newblock \emph{arXiv preprint arXiv:2310.14189}, 2023.

\bibitem[Song et~al.(2020)Song, Sohl-Dickstein, Kingma, Kumar, Ermon, and Poole]{sde}
Yang Song, Jascha Sohl-Dickstein, Diederik~P Kingma, Abhishek Kumar, Stefano Ermon, and Ben Poole.
\newblock Score-based generative modeling through stochastic differential equations.
\newblock \emph{arXiv preprint arXiv:2011.13456}, 2020.

\bibitem[Song et~al.(2023)Song, Dhariwal, Chen, and Sutskever]{song2023consistency}
Yang Song, Prafulla Dhariwal, Mark Chen, and Ilya Sutskever.
\newblock Consistency models.
\newblock \emph{arXiv preprint arXiv:2303.01469}, 2023.

\bibitem[Unterthiner et~al.(2018)Unterthiner, Van~Steenkiste, Kurach, Marinier, Michalski, and Gelly]{unterthiner2018towards}
Thomas Unterthiner, Sjoerd Van~Steenkiste, Karol Kurach, Raphael Marinier, Marcin Michalski, and Sylvain Gelly.
\newblock Towards accurate generative models of video: A new metric \& challenges.
\newblock \emph{arXiv preprint arXiv:1812.01717}, 2018.

\bibitem[Wang et~al.(2024{\natexlab{a}})Wang, Huang, Bergman, Shen, Gao, Lingelbach, Sun, Bian, Song, Liu, et~al.]{wang2024phased}
Fu-Yun Wang, Zhaoyang Huang, Alexander~William Bergman, Dazhong Shen, Peng Gao, Michael Lingelbach, Keqiang Sun, Weikang Bian, Guanglu Song, Yu Liu, et~al.
\newblock Phased consistency model.
\newblock \emph{arXiv preprint arXiv:2405.18407}, 2024{\natexlab{a}}.

\bibitem[Wang et~al.(2024{\natexlab{b}})Wang, Huang, Shi, Bian, Song, Liu, and Li]{wang2024animatelcm}
Fu-Yun Wang, Zhaoyang Huang, Xiaoyu Shi, Weikang Bian, Guanglu Song, Yu Liu, and Hongsheng Li.
\newblock Animatelcm: Accelerating the animation of personalized diffusion models and adapters with decoupled consistency learning.
\newblock \emph{arXiv preprint arXiv:2402.00769}, 2024{\natexlab{b}}.

\bibitem[Wang et~al.(2023)Wang, Zhang, Zhang, Liu, Zhang, Gao, and Sang]{wang2023videolcm}
Xiang Wang, Shiwei Zhang, Han Zhang, Yu Liu, Yingya Zhang, Changxin Gao, and Nong Sang.
\newblock Videolcm: Video latent consistency model.
\newblock \emph{arXiv preprint arXiv:2312.09109}, 2023.

\bibitem[Xu et~al.(2024)Xu, Zhao, Xiao, and Hou]{xu2024ufogen}
Yanwu Xu, Yang Zhao, Zhisheng Xiao, and Tingbo Hou.
\newblock Ufogen: You forward once large scale text-to-image generation via diffusion gans.
\newblock In \emph{Proceedings of the IEEE/CVF Conference on Computer Vision and Pattern Recognition}, pages 8196--8206, 2024.

\bibitem[Yan et~al.(2024)Yan, Liu, Pan, Liew, Liu, and Feng]{yan2024perflow}
Hanshu Yan, Xingchao Liu, Jiachun Pan, Jun~Hao Liew, Qiang Liu, and Jiashi Feng.
\newblock Perflow: Piecewise rectified flow as universal plug-and-play accelerator.
\newblock \emph{arXiv preprint arXiv:2405.07510}, 2024.

\bibitem[Yin et~al.(2024{\natexlab{a}})Yin, Gharbi, Park, Zhang, Shechtman, Durand, and Freeman]{yin2024improved}
Tianwei Yin, Micha{\"e}l Gharbi, Taesung Park, Richard Zhang, Eli Shechtman, Fredo Durand, and William~T Freeman.
\newblock Improved distribution matching distillation for fast image synthesis.
\newblock \emph{arXiv preprint arXiv:2405.14867}, 2024{\natexlab{a}}.

\bibitem[Yin et~al.(2024{\natexlab{b}})Yin, Gharbi, Zhang, Shechtman, Durand, Freeman, and Park]{yin2024one}
Tianwei Yin, Micha{\"e}l Gharbi, Richard Zhang, Eli Shechtman, Fredo Durand, William~T Freeman, and Taesung Park.
\newblock One-step diffusion with distribution matching distillation.
\newblock In \emph{Proceedings of the IEEE/CVF Conference on Computer Vision and Pattern Recognition}, pages 6613--6623, 2024{\natexlab{b}}.

\bibitem[Zhang et~al.(2024)Zhang, Li, Wu, Xu, Kag, Skorokhodov, Menapace, Siarohin, Cao, Metaxas, et~al.]{zhang2024sf}
Zhixing Zhang, Yanyu Li, Yushu Wu, Yanwu Xu, Anil Kag, Ivan Skorokhodov, Willi Menapace, Aliaksandr Siarohin, Junli Cao, Dimitris Metaxas, et~al.
\newblock Sf-v: Single forward video generation model.
\newblock \emph{arXiv preprint arXiv:2406.04324}, 2024.

\bibitem[Zheng et~al.(2024)Zheng, Hu, Fan, Wang, Ding, Tao, and Cham]{zheng2024trajectory}
Jianbin Zheng, Minghui Hu, Zhongyi Fan, Chaoyue Wang, Changxing Ding, Dacheng Tao, and Tat-Jen Cham.
\newblock Trajectory consistency distillation.
\newblock \emph{arXiv preprint arXiv:2402.19159}, 2024.

\bibitem[Zhou et~al.(2024)Zhou, Zheng, Wang, Yin, and Huang]{zhou2024score}
Mingyuan Zhou, Huangjie Zheng, Zhendong Wang, Mingzhang Yin, and Hai Huang.
\newblock Score identity distillation: Exponentially fast distillation of pretrained diffusion models for one-step generation.
\newblock In \emph{ICML}, 2024.

\end{thebibliography}
}

\appendix

\newtheorem{customthm}{Theorem}[section]
\DeclarePairedDelimiter{\norm}{\lVert}{\rVert}

\clearpage
\onecolumn
\setcounter{page}{1}
\section*{\Huge Appendix}

\startcontents[appendices]
\printcontents[appendices]{l}{1}{\setcounter{tocdepth}{2}}


\clearpage
\section{Proofs}\label{app:proof}

The following is based on consistency distillation~\cite{song2023consistency}.
\subsection{Multi-Step Solving Method}\label{app:proof_cd}
\begin{customthm}
Let $\Delta t \coloneqq \max_{n \in \llbracket 1, N-1\rrbracket}\{|t_{n+1} - t_{n}|\}$, and $f(\cdot,\cdot;\phi)$ be the target phased consistency function induced by the pre-trained diffusion
 model (empirical PF-ODE). Assume $f_\theta$ satisfies the Lipschitz condition: there exists $L > 0$ such that for all $t \in [\epsilon, T]$, $x$, and $y$, we have $\norm{f_\theta(x, t) - f_\theta(y, t)}_2 \leq L \norm{x - y}_2$. Assume further that for all $n \in \llbracket 1, N-1 \rrbracket$, the ODE solver called at $t_{n+1}$ has local error uniformly bounded by $O((t_{n+1} - t_n)^{p+1})$ with $p\geq 1$. Then, if $Dis(f_\theta(x_{t_{n+m}}, t_{n+m}),f_{\theta}(\hat{x}_{t_n}^\phi, t_n)) = 0$, we have
\begin{align*}
    \sup_{n, x}\|f_{\theta}(x, t_n) - f(x, t_n; \phi)\|_2 = O((\Delta t)^p).
\end{align*}
\end{customthm}
\begin{proof}
    From the loss $Dis(f_\theta(x_{t_{n+m}}, t_{n+m}),f_{\theta}(\hat{x}_{t_n}^\phi, t_n)) = 0$, we have:
    \begin{align}
        f_\theta(x_{t_{n+m}}, t_{n+m}) \equiv f_{\theta}(\hat{x}_{t_n}^\phi, t_n).\label{eq:zero_loss_identity}
    \end{align}
Let $e_{n} \coloneqq f_\theta(x_{t_{n}}, t_{n}) - f(x_{t_n}, t_n; \phi).$
    We obtain the subsequent recursive formula:
    \begin{align}
        e_{n+m} &= f_\theta(x_{t_{n+m}}, t_{n+m}) - f(x_{t_{n+m}}, t_{n+m}; \phi)\notag\\
        &\stackrel{(i)}{=} f_\theta(\hat{x}_{t_{n}}^\phi, t_{n}) - f(x_{t_{n}}, t_{n}; \phi)\notag\\
        &= f_\theta(\hat{x}_{t_{n}}^\phi, t_{n}) - f_\theta(x_{t_n}, t_n) + f_\theta(x_{t_n}, t_n) - f(x_{t_{n}}, t_{n}; \phi)\notag\\
        &= f_\theta(\hat{x}_{t_{n}}^\phi, t_{n}) - f_\theta(x_{t_n}, t_n) + e_{n},
    \end{align}
    where (i) is due to \cref{eq:zero_loss_identity} and $f(x_{t_{n+m}}, t_{n+m}; \phi) =  f(x_{t_{n}}, t_{n}; \phi)$. Considering $f_\theta(\cdot, t_n)$ has Lipschitz constant $L$, we have:
    \begin{align}
       \norm{e_{n+m}}_2 &\leq \norm{e_{n}}_2 + L \norm{\hat{x}_{t_n}^\phi - x_{t_n}}_2\\
        &\stackrel{(i)}{=}  \norm{e_{n}}_2 + L\cdot O(\max_{k \in \llbracket n, n+m-1\rrbracket}(t_{k+1} - t_k)^{p+1})\\
        &=\norm{e_{n}}_2 + O(\max_{k \in \llbracket n, n+m-1\rrbracket}(t_{k+1} - t_k)^{p+1}).
    \end{align}
Considering the definition of \( f \), we have:
    \begin{align}
        e_0 &= f_\theta(x_{t_0}, t_0) - f(x_{t_0}, t_0; \phi) \\
        &\stackrel{(ii)}{=} x_{t_0} - x_{t_0}\\
        &= \bm{0}.
    \end{align}
    Let $j*m==N$, we have:
    \begin{align}
        \norm{e_{m*j}}_2 &\leq \norm{e_{0}}_2 + \sum_{k=0}^{j-1} O(\max_{l \in \llbracket k*m, (k+1)*m-1\rrbracket}(t_{l+1} - t_{l})^{p+1}) \\
        &= \sum_{k=0}^{j-1} O(\max_{l \in \llbracket k*m, (k+1)*m-1\rrbracket}(t_{l+1} - t_{l})^{p+1})\\
        &= \sum_{k=0}^{j-1} (\max_{l \in \llbracket k*m, (k+1)*m-1\rrbracket}(t_{l+1} - t_{l})) O(\max_{l \in \llbracket k*m, (k+1)*m-1\rrbracket}(t_{l+1} - t_{l})^{p})\\
        &\leq \sum_{k=1}^{j-1} (T-\epsilon) O(\max_{l \in \llbracket k*m, (k+1)*m-1\rrbracket}(t_{l+1} - t_{l})^{p})\\
        &\leq \sum_{k=1}^{j-1} (T-\epsilon) O((\Delta t)^{p})\\
        &= O((\Delta t)^p)
    \end{align}
    which completes the proof. Eq. 24 and Eq. 25 demonstrate that our method has a smaller error upper bound.
\end{proof}

\begin{table}[h]
  \caption{MSE Loss of Feature Extracted by DINOv2 During LGP and ACD Stages.}
  \label{tab:DINOv2loss}
  \centering

  \begin{tabular}{ccc}
    \toprule
Stage &MSE(DINOv2($x^\text{Image}_{in}$),DINOv2($x^\text{Predict}$)) &MSE(DINOv2($f_\theta(x_{t_{n+m}},t_{n+m})$),DINOv2($x^\text{Predict}$))\\	\midrule
LGP  &0.21  &0.26  \\ 
ACD &0.0022 &4.09e-5 \\ 
    \bottomrule
  \end{tabular}

\end{table}
\section{Discussion}
\subsection{Discussion on LGP and ACD}\label{app:discussion}
We demonstrate the convergence of training at different stages based on PCM~\cite{wang2024phased}. 
Let the data distribution used in the LGP and ACD phases be denoted as \( p_0 \), and the forward conditional probability path is defined as \( \alpha_{t} \mathbf{x}_0 + \sigma_{t} \boldsymbol{\epsilon} \). The intermediate distribution is then defined as \( {p}_{t} (\mathbf{x}) = ({p}_{0}(\frac{\mathbf{x}}{\alpha_{t}}) \cdot \frac{1}{\alpha_t}) \ast \mathcal{N}(0, \sigma_{t}) \). Similarly, the data distribution used for pretraining the diffusion model is denoted as \( {p}^{\text{pretrain}}_0(\mathbf{x}) \), and the corresponding intermediate distribution during the forward process is \( {p}^{\text{pretrain}}_{t}(\mathbf{x}) = ({p}^{\text{pretrain}}_{0}(\frac{\mathbf{x}}{\alpha_{t}}) \cdot \frac{1}{\alpha_t}) \ast \mathcal{N}(0, \sigma_{t}) \). This is reasonable because current large diffusion models are typically trained with more resources on larger datasets compared to those used for consistency distillation. We denote \( \mathcal{T}^{\boldsymbol{\phi}}_{t \to s} \), \( \mathcal{T}^{\boldsymbol{\theta}}_{t \to s} \), and \( \mathcal{T}^{\phi'}_{t \to s} \) as the flow operators corresponding to the pre-trained diffusion model, the flow operators corresponding to our consistency model, and the PF-ODE of the data distribution used for consistency distillation, respectively. 

We first discuss the convergence of \( \mathcal{L}_{\text{ACD}}^{adv} \). We have \( f_\theta(x_{t_{n+m}}, t_{n+m}) \equiv f_{\theta}(\hat{x}_{t_n}^\phi, t_n) \), where \( x_{t_{n+m}} \in p_{n+m} \) and \( x_{t_{n}} \in p_{n} \). Consequently, we obtain:
\begin{align}
    \mathcal{T}^{\boldsymbol{\theta}}_{t_{n+m} \to \epsilon} \# \mathbb{P}_{t_{n+m}} \equiv \mathcal{T}^{\boldsymbol{\theta}}_{t_{n} \to \epsilon} \mathcal{T}^{\boldsymbol{\phi}}_{t_{n+m} \to t_{n}} \# \mathbb{P}_{t_{n+m}}\, .
\end{align}
Therefore, if $Dis(f_\theta(x_{t_{n+m}}, t_{n+m}),f_{\theta}(\hat{x}_{t_n}^\phi, t_n)) = 0$, we have $\mathcal{L}_{\text{ACD}}^{adv}=0$.

We discuss the convergence of \( \mathcal{L}_{\text{LGP}}^{adv} \). We have:
\begin{equation}
    p_{0} \equiv \mathcal T^{\boldsymbol \phi'}_{t_{n+m}\to 0} \# p_{t_{n+m}} \, .
\end{equation}
Therefore, we have
\begin{align}
    &Dis\left(\mathcal T^{\boldsymbol \theta}_{t_{n+m}\to \epsilon}\#p_{t_{n+m}}\middle\| p_0 \right) \, \\
    =& Dis\left(\mathcal T^{\boldsymbol \theta}_{t_{n+m}\to \epsilon}\#p_{t_{n+m}}\middle\|   \mathcal T^{\boldsymbol \phi'}_{t_{n+m}\to 0} \# p_{t_{n+m}} \right)
\end{align}
Because \( f_\theta(x_{t_{n+m}}, t_{n+m}) \equiv f_{\theta}(\hat{x}_{t_n}^\phi, t_n) \), we have:
\begin{align}
    & Dis\left(\mathcal T^{\boldsymbol \theta}_{t_{n+m}\to \epsilon}\#p_{t_{n+m}}\middle\|   \mathcal T^{\boldsymbol \phi'}_{t_{n+m}\to 0} \# p_{t_{n+m}} \right) \, \\
    =& Dis\left(\mathcal T^{\boldsymbol \phi}_{t_{n+m}\to \epsilon}\#p_{t_{n+m}}\middle\|   \mathcal T^{\boldsymbol \phi'}_{t_{n+m}\to 0} \# p_{t_{n+m}} \right) \\
    =& Dis\left(p_0^{\text{pretrain}}\middle\|   p_0 \right) 
\end{align}
Because $p^\text{pretrain}_{0} \not=p_{0}$, we have $\mathcal{L}_{\text{LGP}}^{adv}>0$.

We consider the input condition \( x_{\text{in}}^{\text{Image}} \) for the diffusion model, which involves replicating the image condition across multiple frames to align with the frame count of the original video. The output of our consistency model is \( f_\theta(x_{t_{n+m}}, t_{n+m}) \). During the LGP phase, our prediction target is \( x^{\text{Predict}} = x_0 \). During the ACD phase, our prediction target is \( x^{\text{Predict}} = f_{\theta^{-}}(\hat{x}^\phi_{t_n}, t_n) \).

We extract features from these data using DINOv2 and compute the MSE loss of these features. As shown in Table~\ref{tab:DINOv2loss}, during the LGP phase, the difference between \( x_{\text{in}}^{\text{Image}} \) and \( x^{\text{Predict}} \) is minimal, indicating that our consistency model tends to predict multiple static images. During the ACD phase, the difference between \( f_\theta(x_{t_{n+m}}, t_{n+m}) \) and \( x^{\text{Predict}} \) is minimal, indicating that our consistency model tends to predict data generated by the pre-trained model. Although random noise is added to \( x_{\text{in}}^{\text{Image}} \) in actual training, this does not fundamentally solve the issue. However, fortunately, using LGP in the early stage of model training can accelerate the convergence of our distillation model. Figure~\textcolor{blue}{\ref{fig:teaser1}} demonstrates the effectiveness of using LGP initially.

\subsection{Contributions}\label{app:contributions}
Here, we re-emphasize the key components of OSV and summarize the contributions of our work.

The primary motivation of this research is to expedite the sampling process for high-resolution image-to-video generation by leveraging the consistency model training paradigm.  Previous methods, including Animate-LCM and SF-V, sought to harness the potential of consistency models in this demanding scenario but failed to deliver satisfactory outcomes.  We systematically examine and dissect the limitations of these approaches from three distinct perspectives.  Crucially, these methods largely represent direct extensions of techniques originally devised to accelerate text-to-image sampling, and their straightforward adaptation to image-to-video sampling introduces significant challenges.  To address these issues, we broaden the design space and propose comprehensive solutions to overcome these limitations.

The OSV framework is built upon the decomposition of the training process into two distinct stages, each utilizing a tailored distillation method to ensure efficient and effective model training.  In the second stage, we introduce a multi-step solving method that capitalizes on the teacher model to execute multiple reverse ODE processes, thereby enhancing prediction accuracy.  As illustrated in Figure~\ref{fig:teaser}, this multi-step solving method not only accelerates training but also significantly improves the performance of the consistency model.

Furthermore, inspired by the inherent properties of consistency models, we propose a novel higher-order solver, termed TTS, to replace the conventional CFG method.  Experimental evaluations substantiate the efficacy of TTS, with results demonstrating state-of-the-art image-to-video generation performance.  Remarkably, our approach achieves this using only 8 H800 GPUs (with merely 2 H800 GPUs required in the second stage), underscoring the efficiency and effectiveness of the proposed method.

\subsection{Removing CFG}
We introduce CFG into the distilled model: $\hat{\Phi}(\mathbf{x}_{t_{n+m}}, t_{n+m}, c; \phi) = {\Phi}(\mathbf{x}_{t_{n+m}}, t_{n+m}, c_{zero}; \phi) + w * ({\Phi}(\mathbf{x}_{t_{n+m}}, t_{n+m}, c; \phi) - {\Phi}(\mathbf{x}_{t_{n+m}}, t_{n+m}, c_{zero}; \phi)).$ This means the model already has CFG during inference, and using the same CFG scale again during inference leads to exposure issues in the generated videos. Table 2f also shows that a smaller CFG scale does not significantly improve the video quality. Removing CFG not only speeds up the model generation but also improves the overall quality of the generated videos.

\section{Additional Experimental Settings}
$\lambda^{LGP}$ and $\lambda^{ACD}$ are set to 0.1. In the Huber Loss, we set $c = 0.001$.

We train the model with videos of 14 frames, and the test videos also consist of 14 frames.

We use TTS only when the step equals 1.




\section{Upsampling Module}
As shown in Figure\textcolor{red}{~\ref{fig:upsample}}, the upsampling module is displayed. First, we increase the number of channels of the latent space features, and then upsample the latent space features using the PixelShuffle operation. We set $r = 4$.


\begin{figure}[t]
\centering
\includegraphics[width=0.2\linewidth]{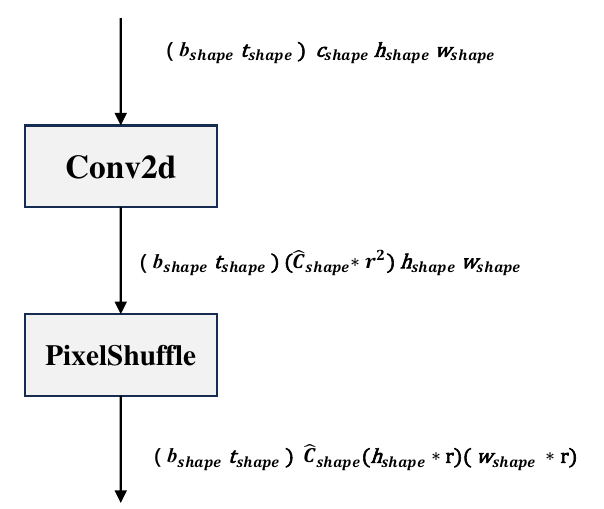}
\caption{Upsampling Module. We design the upsampling module inspired by sub-pixel convolution~\cite{shi2016real}. }
\label{fig:upsample}
\end{figure}





\end{document}